\documentclass[10pt,twocolumn,letterpaper]{article}

\usepackage{cvpr}
\usepackage{times}
\usepackage{epsfig}
\usepackage{graphicx}
\usepackage{amsmath}
\usepackage{amssymb}

\usepackage{algorithm}
\usepackage{algorithmic}
\usepackage{bm}
\usepackage{amsthm}
\usepackage{multirow, paralist}

\def \R {\mathbb{R}}

\def \P {\mathcal{P}}
\def \x {\mathbf{x}}

\def \X {\mathbf{X}}

\def \z {\mathbf{z}}

\def \p {\mathbf{p}}
\def \q {\mathbf{q}}

\def \D {\mathcal{D}}

\def \LL {\mathcal{L}}

\newtheorem{thm}{Theorem}
\newtheorem{prop}{Proposition}

\newtheorem{definition}{Definition}

\usepackage[breaklinks=true,bookmarks=false]{hyperref}

\cvprfinalcopy 

\def\cvprPaperID{9807} 

\ifcvprfinal\fi
\begin{document}

\title{Hierarchically Robust Representation Learning}

\author{Qi Qian$^1$\quad Juhua Hu$^2$ \quad Hao Li$^1$\\
$^1$Alibaba Group\\
$^2$School of Engineering and Technology\\
University of Washington, Tacoma, USA\\
{\tt\small \{qi.qian, lihao.lh\}@alibaba-inc.com\quad juhuah@uw.edu}
}

\maketitle
\thispagestyle{empty}

\begin{abstract}
With the tremendous success of deep learning in visual tasks, the representations extracted from intermediate layers of learned models, that is, deep features, attract much attention of researchers. Previous empirical analysis shows that those features can contain appropriate semantic information. Therefore, with a model trained on a large-scale benchmark data set (e.g., ImageNet), the extracted features can work well on other tasks. In this work, we investigate this phenomenon and demonstrate that deep features can be suboptimal due to the fact that they are learned by minimizing the empirical risk. When the data distribution of the target task is different from that of the benchmark data set, the performance of deep features can degrade. Hence, we propose a hierarchically robust optimization method to learn more generic features. Considering the example-level and concept-level robustness simultaneously, we formulate the problem as a distributionally robust optimization problem with Wasserstein ambiguity set constraints, and an efficient algorithm with the conventional training pipeline is proposed. Experiments on benchmark data sets demonstrate the effectiveness of the robust deep representations.
\end{abstract}

\section{Introduction}
\label{sec:intro}
Extracting appropriate representations is essential for visual recognition. In the past decades, various hand-crafted features have been developed to capture semantics of images, e.g., SIFT~\cite{Lowe04}, HOG~\cite{DalalT05}, etc. The conventional pipeline works in two phases. In the first phase, representations are extracted from each image with a given schema. Thereafter, a specific model (e.g., SVM~\cite{CortesV95}) is learned with these features for a target task. Since the hand-crafted features are task-independent, the performance of this pipeline can be suboptimal.

\begin{figure}[!ht]
\centering
\includegraphics[width=3.2in]{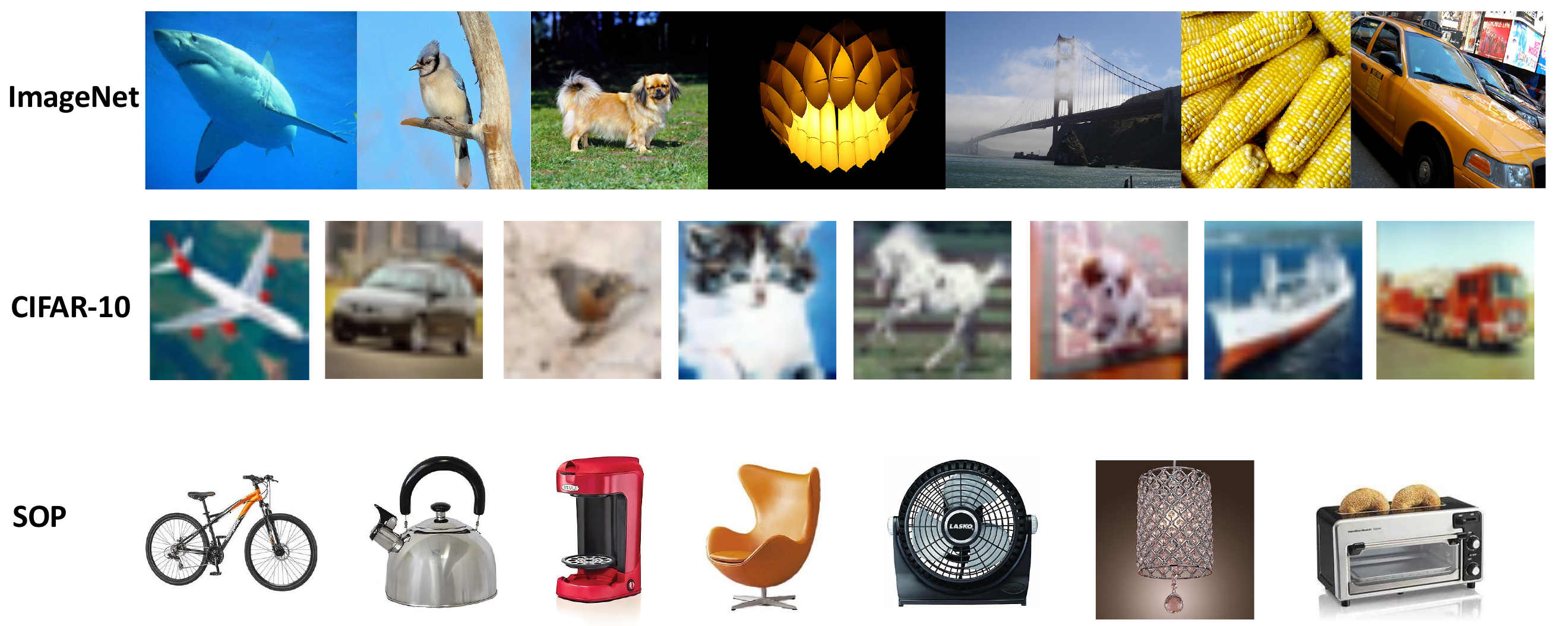}
\caption{Examples from ImageNet, CIFAR-10 and SOP. Example-level distribution difference within a class can be observed from the 7th image of ImageNet and 2nd image of CIFAR-10 for the car class in various aspects, e.g., resolution and pose. Concept-level distribution difference is significant between ImageNet and SOP. ImageNet includes many classes from the concept ``animal'' while SOP only contains classes from ``artifact''.}\label{fig:example}
\end{figure}

Deep learning proposes to incorporate these phases by training end-to-end convolutional neural networks. Without an explicit feature design like SIFT~\cite{Lowe04}, a task-dependent representation will be learned through multiple layers and a fully connected layer is attached at the end as a linear classifier for recognition. Benefited from this coherent structure, deep learning promotes the performance on visual tasks dramatically, e.g., categorization~\cite{KrizhevskySH12}, detection~\cite{RenHG017}, etc. Despite the success of deep learning on large-scale data sets, deep neural networks (DNNs) are easy to overfit small data sets due to the large number of parameters. Besides, DNNs require GPU for efficient training, which is expensive.

Researchers attempt to leverage pre-trained DNNs to improve the feature design mechanism. Surprisingly, it is observed that the features extracted from the last few layers perform well on the generic tasks when the model is pre-trained on a large-scale benchmark data set, e.g., ImageNet~\cite{ILSVRC15}. Deep features, which are outputs from intermediate layers of a deep model, become popular as the substitute of training deep models for light computation. Systematic comparison shows that these deep features outperform the existing hand-crafted features with a large margin~\cite{DonahueJVHZTD14,mormont2018comparison,QianJZL15}. 


The objective of learning deep models for specific tasks and deep features for generic tasks can be different, but little efforts have been devoted to further investigating deep features. When learning deep models, it focuses on optimizing the performance on the current training data set. In contrast, deep features should be learned for generic tasks rather than a single data set. In the applications of deep features, it is also noticed that the deep features can fail when the data distribution in a generic task is different from the benchmark data set~\cite{ZhouLXTO14}. By studying the objective of learning models for a given task, we find that it is a standard empirical risk minimization (ERM) problem that is optimized on the uniform distribution over examples. It is well known that the models obtained by ERM can generalize well on the data from the same distribution as training~\cite{BousquetE02}.

However, the data distribution from real applications can be significantly different from a benchmark data set, which can result in the performance degeneration when adopting the representations learned from ERM. The differences can come from at least two aspects. First, the distribution of examples in each class can be different between the generic task and the benchmark data set, which is referred as example-level distribution difference in this paper. Taking the 7th image of ImageNet and 2nd of CIFAR-10 in Fig.~\ref{fig:example} as an example, they are of different resolutions and poses while they are both from the car class. This problem attracts much attention recently and some approaches to optimize the worst-case performance are developed to handle this issue~\cite{ChenLSS17,NamkoongD16,sinha2018certifiable}. Second, the distribution of concepts in an application is also different from that in the benchmark data set. It should be noted that each concept here can contain multiple classes, e.g., bulldog, beagle and so on under the concept ``dog''. This concept-level distribution difference has been less investigated but more crucial for deploying deep features due to the fact that the concepts in real applications may be only a subset of or partially overlapped by those in the benchmark data set. For instance, the concepts in SOP is quite different from those covered in ImageNet as shown in Fig.~\ref{fig:example}.


In this work, we propose to consider the difference in examples and that in concepts simultaneously and learn hierarchically robust representations from DNNs. Compared with ERM, our algorithm is more consistent with the objective of learning generic deep features. For the example-level robustness, we adopt Wasserstein ambiguity set~\cite{sinha2018certifiable} to encode the uncertainty from examples for the efficient optimization. Our theoretical analysis also illustrates that an appropriate augmentation can be better than the regularization in training DNNs, since the former one provides a tighter approximation for the optimization problem. For the concept-level robustness, we formulate it as a game between the deep model and the distribution over different concepts to optimize the worst-case performance over concepts. By learning deep features with the adversarial distribution, the worst-case performance over concepts can be improved. Finally, to keep the simplicity of the training pipeline, we develop an algorithm that leverages the standard random sampling strategy at each iteration and re-weights the obtained gradient for an unbiased estimation. This step may increase the variance of the gradient and we reduce the variance by setting the learning rate elaborately. We show that the adversarial distribution can converge at the rate of $\mathcal{O}(\log(T)/T)$, where $T$ denotes the total number of iterations. We employ ImageNet as a benchmark data set for learning deep features and the empirical study on real-world data sets confirms the effectiveness of our method.

The rest of this paper is organized as follows. Section \ref{sec:related} reviews the related work. Section \ref{sec:method} introduces the proposed method. Section \ref{sec:exp} conducts the experiments on the benchmark data sets and Section \ref{sec:conclude} concludes this work with future directions.

\section{Related Work}\label{sec:related}
\noindent\textbf{Deep Features:} Deep learning becomes popular since ImageNet ILSVRC12 and various architectures of DNNs have been proposed, e.g., AlexNet~\cite{KrizhevskySH12}, VGG~\cite{Simonyan14c}, GoogLeNet~\cite{SzegedyLJSRAEVR15}, and ResNet~\cite{HeZRS16}. Besides the success on image categorization, features extracted from the last few layers are applied for generic tasks. \cite{DonahueJVHZTD14} adopts the deep features from the last two layers in AlexNet and shows the impressive performance on visual recognition with different applications. After that, \cite{QianJZL15} applies deep features for distance metric learning and achieves the overwhelming performance to the hand-crafted features on fine-grained visual categorization. \cite{mormont2018comparison} compares deep features from different neural networks and ResNet shows the best results. Besides the model pre-trained on ImageNet, \cite{ZhouLXTO14} proposes to learn deep features with a large-scale scene data set to improve the performance on the scene recognition task. All of these work directly extract features from the model learned with ERM as the objective. In contrast, we develop an algorithm that is tailored to learn robust deep representations. Note that deep features can be extracted from multiple layers of deep models and we focus on the layer before the final fully-connected layer in this work. 

\noindent\textbf{Robust Optimization:} Recently, distributionally robust optimization that aims to optimize the worst-case performance has attracted much attention~\cite{ChenLSS17,NamkoongD16,sinha2018certifiable}. \cite{NamkoongD16} proposes to optimize the performance with worst-case distribution over examples that is derived from the empirical distribution. \cite{ChenLSS17} extends the problem to a non-convex loss function, but they require a near-optimal oracle for the non-convex problem to learn the robust model. \cite{sinha2018certifiable} introduces the adversarial perturbation on each example for robustness. Most of these algorithms only consider the example-level robustness. In contrast, we propose the hierarchically robust optimization that considers the example-level and concept-level robustness simultaneously, to learn the generic deep representations for real applications.

\section{Hierarchical Robustness}
\label{sec:method}
\subsection{Problem Formulation}
Let $\x_i$ denote an image and $y_i\in\{1,\dots,C\}$ be its corresponding label for a $C$-class classification problem. Given a benchmark data set $\{\x_i, y_i\}$ where $i=1,\ldots, N$, the parameter $\theta$ in a deep neural network can be learned by solving the optimization problem as
\begin{eqnarray}\label{eq:erm}
\min_{\theta}\frac{1}{N}\sum_i \ell(\x_i,y_i;\theta)
\end{eqnarray}
where $\ell(\cdot)$ is a non-negative loss function (e.g., cross entropy loss). By decomposing the parameter $\theta$ as $\theta=\{\delta,\omega\}$, where $\omega$ denotes the parameter of the final fully-connected layer and $\delta$ denotes the parameter from other layers and can be considered as for a feature extraction function $f(\cdot)$, we can rewrite the original problem as
\[\min_{\theta}\frac{1}{N}\sum_i \ell(f(\x_i),y_i;\omega)\] 
Considering that $\omega$ is for a linear classifier, which is consistent to the classifiers applied in real-world applications (e.g., SVM), the decomposition shows that the problem of learning generic deep features $f(\x)$ can be addressed by learning a robust deep model on the benchmark data set.

The original problem in Eqn.~\ref{eq:erm} is an empirical risk minimization (ERM) problem that can be inappropriate for learning generic representations. In the following, we explore the hierarchical robustness to obtain robust deep representations for generic tasks.

First, we consider the example-level robustness. Unlike ERM, a robust model is to minimize the loss with the worst-case distribution derived from the empirical distribution. The optimization problem can be cast as a game between the prediction model and the adversarial distribution
\[\min_{\theta} \max_i\{\ell(\x_i,y_i;\theta)\}\]
which is equivalent to
\[\min_{\theta} \max_{\p\in\R^{N};\p\in\Delta} \sum_i p_i \ell(\x_i,y_i;\theta)\]
where $\p$ is the adversarial distribution over training examples and $\Delta$ is the simplex as $\Delta = \{\p|\sum_i p_i=1,\forall i, p_i\geq 0\}$. When $\p$ is a uniform distribution, the distributioanlly robust optimization becomes ERM. 

Without any constraints, the adversarial distribution is sensitive to the outlier and can be arbitrarily far way from the empirical distribution, which has large variance from the selected examples. Therefore, we introduce a regularizer to constrain the space of the adversarial distribution, which provides a trade-off between the bias (i.e., to the empirical distribution) and variance for the adversarial distribution. The problem can be written as
\begin{eqnarray}\label{eq:expprob}
\min_{\theta} \max_{\p\in\R^{N};\p\in\Delta} \sum_i p_i \ell(\x_i,y_i;\theta) - \lambda_e\D(\p||\p_0)
\end{eqnarray}
where $\p_0$ is the empirical distribution. $\D(\cdot)$ measures the distance between the learned adversarial distribution and the empirical distribution. We apply squared $L_2$ distance in this work as $\D(\p||\p_0)=\|\p-\p_0\|_2^2$. The regularizer is to guarantee that the generated adversarial distribution is not too far way from the empirical distribution. It implies that the adversarial distribution is from an ambiguity set as
\[\p\in \{\p:\D(\p||\p_0)\leq \epsilon\}\]
where $\epsilon$ is determined by $\lambda_e$.

Besides the example-level robustness, concept-level robustness is more important for learning the generic features. A desired model should perform consistently well over different concepts. Assuming that there are $K$ concepts in the training set and each concept consists of $N_k$ examples, the concept-robust optimization problem is
\[\min_{\theta}\max_k \{\frac{1}{N_k}\sum_i^{N_k}\ell(\x_i^k,y_i^k;\theta)\}\]

With the similar analysis as the example-level robustness and adopting the appropriate regularizer, the problem becomes
\begin{eqnarray}\label{eq:concept}
\min_{\theta}\max_{\q\in\R^{K};\q\in\Delta} \sum_k \frac{q_k}{N_k}\sum_i^{N_k}\ell(\x_i^k,y_i^k;\theta) -\lambda_c\D(\q||\q_0)
\end{eqnarray}
where $\q_0$ can be set as $q_0^k = N_k/N$.

Combined with the example-level robustness, the hierarchically robust optimization problem becomes
\begin{eqnarray*}
\min_{\theta}\max\limits_{\substack{\p\in\R^{N}; \p\in\Delta  \\ \q\in\R^{K}; \q\in\Delta}}&& \sum_k  \frac{q_k}{N_k}\sum_i^{N_k}p_i\ell(\x_i^k,y_i^k;\theta)\\
&& - \lambda_e \D(\p||\p_0) -\lambda_c\D(\q||\q_0)
\end{eqnarray*}
In this formulation, each example is associated with a parameter $p_i$ and $q_k$. Therefore, a high dimensionality with this coupling structure makes an efficient optimization challenging. Due to the fact that $K\ll N$, we decouple the hierarchical robustness with an alternative formulation for the example-level robustness as follows.

\subsection{Wasserstein Ambiguity Set}

In Eqn.~\ref{eq:expprob}, the ambiguity set is defined with the distance to the uniform distribution over the training set. It introduces the adversarial distribution by re-weighting each example, which couples the parameter with that of the concept-level problem. To simplify the optimization, we generate the ambiguity set for the adversarial distribution with Wasserstein distance~\cite{sinha2018certifiable}. The property of Wasserstein distance can help to decouple the example-level robustness from concept-level robustness.

Assume that $P$ is a data-generating distribution over the data space and $P_0$ is the empirical distribution from where the training set is generated as $\x \sim P_0$. The ambiguity set for the distribution $P$ can be defined as
\[\{P:W(P, P_0)\leq \epsilon\}\]
$W(P, P_0) = \inf_{M\in \Pi(P,P_0)}E_{M}[d(\hat{\x},\x)]$ is the Wasserstein distance between distributions~\cite{sinha2018certifiable} and we denote the example generated from $P$ as $\hat{\x}$. $d(\cdot,\cdot)$ is the transportation cost between examples.

The problem of example-level robustness can be written as
\[\min_{\theta} \max_{P}E_{P}[\ell(\hat{\x},y;\theta)] -\frac{\lambda_w}{2} W(P, P_0)\]
According to the definition of Wasserstein distance~\cite{sinha2018certifiable} and let the cost function be the squared Euclidean distance, the problem is equivalent to
\[\min_{\theta}\max_{\hat{\x}\in\X}\sum_i \ell(\hat{\x}_i,y_i;\theta) - \frac{\lambda_w}{2} \sum_i \|\hat{\x}_i - \x_i\|_F^2\]
where $\X$ is the data space. In \cite{sinha2018certifiable}, they obtain the optimal $\hat{\x}_i$ by solving the subproblem for each example at each iteration. To accelerate the optimization, we propose to minimize the upper bound of the subproblem, which also provides an insight for the comparison between regularization and augmentation.

The main theoretical results are stated in the following theorems and 
their proofs can be found in the appendix. First, we give the definition of smoothness as
\begin{definition}
A function $f$ is called $L_z$-smoothness in $z$ w.r.t. a norm $\|\cdot\|$ if there is a constant $L_z$ such that for any values of $z$ as $z'$ and $z''$, it holds that
\[f(z'')\leq f(z')+\langle \nabla f(z'),z''-z'\rangle+\frac{L_z}{2}\|z''-z'\|^2 \]
\end{definition}

\begin{thm}\label{thm:1}
Assuming $\ell(\cdot)$ is $L_{\x}$-smoothness in $\x$ and $\nabla_{\x} \ell$ is $L_{\theta}$-Lipschitz continuous for $\theta$, we have
\[\max_{\hat{\x}_i\in \X} \ell(\hat{\x}_i,y_i;\theta) - \frac{\lambda_w}{2} \|\hat{\x}_i - \x_i\|_F^2\leq \ell(\x_i,y_i;\theta)+\frac{\gamma}{2}\|\theta\|_F^2\]
where $\lambda_w$ is sufficiently large such that $\lambda_w>L_{\x}$ and $ \gamma = \frac{L_{\theta}^2}{\lambda_w -L_{\x}}$.
\end{thm}

\begin{thm}\label{thm:2}
With the same assumption in Theorem~\ref{thm:1} and considering an additive augmentation with $\z$ for the original image
\[\tilde{\x}_i = \x_i+\tau \z_i\]
we have
\[\max_{\hat{\x}_i\in \X}  \ell(\hat{\x}_i,y_i;\theta) - \frac{\lambda_w}{2} \|\hat{\x}_i - \x_i\|_F^2\leq \ell(\tilde{\x}_i,y_i;\theta)+\frac{\gamma}{2}\|\theta\|_F^2 - \alpha\]
where 
\[\tau = \frac{\langle \nabla_{\x_i}\ell,\z_i\rangle}{3L_{\x}\|\z_i\|_F^2}\]
and $\alpha$ is a non-negative constant as
\[\alpha =  \frac{\lambda_w}{\lambda_w-L_{\x}}\frac{\langle \nabla_{\x_i}\ell,\z_i\rangle^2}{6L_{\x}\|\z_i\|_F^2}\]
\end{thm}

Theorem~\ref{thm:1} shows that learning the model using the original examples with a regularization on the complexity of the model, e.g., weight decay with $\gamma$, can make the learned model robust for examples from the ambiguity set. A similar result has been observed in the conventional robust optimization~\cite{ben2009robust}. However, the regularization is not sufficient to train good enough DNNs and many optimization algorithms have to rely on augmented examples to obtain models with better generalization performance. 

Theorem~\ref{thm:2} interprets the phenomenon by analyzing a specific augmentation that adds a patch $\z$ to the original image and shows that augmented examples can provide a tighter bound for the loss of the examples in the ambiguity set. Besides, the augmented patch $\z_i$ is corresponding to the gradient of the original example $\x_i$. To make the approximation tight, it should be identical to the direction of the gradient. So we set $\z_i = \frac{\nabla_{\x_i}\ell}{\|\nabla_{\x_i}\ell\|_F}$, which is similar to that in adversarial training~\cite{GoodfellowSS14}.

Combining with the concept-level robustness in Eqn.~\ref{eq:concept}, we have the final objective for learning the hierarchically robust representations as
\begin{eqnarray}\label{eq:problem}
\min_{\theta} \max\limits_{\q\in\R^{K}; \q\in\Delta} &&\LL(\q,\theta)=\sum_k\frac{q_k}{N_k}\sum_i\ell(\tilde{\x}_i^k,y_i^k;\theta) \nonumber\\
&&+\frac{\gamma}{2}\|\theta\|_F^2-\frac{\lambda}{2}\|\q-\q_0\|_2^2
\end{eqnarray}

\subsection{Efficient Optimization}
The problem in Eqn.~\ref{eq:problem} can be solved efficiently by stochastic gradient descent (SGD). In the standard training pipeline for ERM in Eqn.~\ref{eq:erm}, a mini-batch of examples are randomly sampled at each iteration and the model is updated with gradient descent as
\[\theta_{t+1} = \theta_t-\eta_{\theta}\frac{1}{m}\sum_i^m\nabla_{\theta}\ell(\x_i,y_i;\theta_t) \]
where $m$ is the size of a mini-batch.

For the problem in Eqn.~\ref{eq:problem}, each example has a weight as $q_k/N_k$ and the gradient has to be weighted for an unbiased estimation as
\begin{eqnarray}\label{eq:model}
\theta_{t+1} = \theta_t - \eta_{\theta}(\frac{1}{m}\sum_i^m \frac{N }{N_k}q_k\nabla_{\theta}\ell(\tilde{\x}_i^{k},y_i^{k};\theta_t)+\gamma\theta_t)
\end{eqnarray}

For the adversarial distribution $\q$, each concept has a weight $q_k$ and the straightforward way is to sample a mini-batch of examples from each concept to estimate the gradient of the distribution. However, the number of concepts varies and it can be larger than the size of a mini-batch. Besides, it results in the different sampling strategies for computing the gradient of deep models and the adversarial distribution, which increases the complexity of the training system. To address the issue, we take the same random sampling pipeline and update the distribution with weighted gradient ascent as
\begin{align}\label{eq:update}
&\hat{q}_k^{t+1} = q_k^t+\eta_q^t\big(\frac{1}{m}\sum_j^{m_k} \frac{N}{N_k}\ell(\tilde{\x}_j^{k},y_j^{k};\theta_t)-\lambda(q_k^t-q_{0}^k)\big)\nonumber\\
&\q_{t+1} = \P_{\Delta}(\hat{\q}_{t+1})
\end{align}
where $m_k$ is the number of examples from the $k$-th concept in the mini-batch and $\sum_k m_k=m$. $\P_{\Delta}(\cdot)$ projects the vector onto the simplex as in \cite{DuchiSSC08}.

The re-weighting strategy makes the gradient unbiased but introduces the additional variance. Since batch-normalization~\cite{IoffeS15} is inapplicable for the parameters of the adversarial distribution that is from the simplex, we develop a learning strategy to reduce the variance from gradients.

First, to illustrate the issue, let $\delta_1$ and $\delta_2$ be two binary random variables as
\[\Pr\{\delta_1=1\} = \frac{1}{N_k};\quad \Pr\{\delta_2=1\} = \frac{1}{N}\]
Obviously, we have $E[\delta_1] = \frac{1}{N_k};\quad E[\frac{N\delta_2}{N_k}] = \frac{1}{N_k}$.
It demonstrates that the gradient after re-weighting is unbiased. However, the variance can be different as
\[\mathrm{Var}[\delta_1] = \frac{1}{N_k} - \frac{1}{N_k^2};\quad \mathrm{Var}[\frac{N\delta_2}{N_k}] = \frac{N}{N_k^2} - \frac{1}{N_k^2}\]
where the variance is roughly increased by a factor of $N/N_k$.

By investigating the updating criterion in Eqn.~\ref{eq:update}, we find that the gradient is rescaled by the learning rate $\eta_q^t$. If we let $\eta_q^t=\mathcal{O}(\frac{1}{t})$, the norm of the gradient will be limited after a sufficient number of iterations. Besides, for any distribution $\q'$, the norm of $\|\q-\q'\|_2^2$ is bounded by a small value of $2$ since the distribution is from the simplex. It inspires us to deal with the first several iterations by adopting a small learning rate. The algorithm is summarized in Alg.~\ref{alg:1}. In short, we use the learning rate as $\eta_t = \frac{1}{c \lambda t}$ where $c>1$ for the first $s$ iterations and then the conventional learning rate $\eta_t = \frac{1}{\lambda t}$ is applied.

\begin{algorithm}[h]
\caption{Hierarchically Robust Representation Learning (HRRL)}
\begin{algorithmic}[1]
\STATE {\bf Input:} Dataset $\{\x_i,y_i\}$, iterations $T$, mini-batch size $m$, $\lambda$, $\gamma$, $\tau$, $s$, $c$
\FOR{$t = 1,\cdots,T$}
\IF{$t<=s$}
\STATE $\eta_q^t = \frac{1}{c\lambda t}$
\ELSE
\STATE $\eta_q^t = \frac{1}{\lambda t}$
\ENDIF
\STATE Sample a mini-batch of examples $\{\x_i,y_i\}_{i=1,\dots,m}$
\STATE Generate the augmented data as $\tilde{\x}_i = \x_i+\tau \z_i$
\STATE Update model with gradient descent as in Eqn.~\ref{eq:model}
\STATE Update distribution with gradient ascent as in Eqn.~\ref{eq:update}
\ENDFOR
\RETURN A feature extraction function $f(\cdot)$ from $\theta_T$
\end{algorithmic}\label{alg:1}
\end{algorithm}

The convergence about the adversarial distribution is stated as follows.
\begin{thm}\label{thm:converge}
Assume the gradient of the adversarial distribution $\q$ is bounded as $\forall t, \|g_q^t\|_2\leq \mu$ and set the learning rate as
\[\eta_q^t =\left\{\begin{array}{ll} \frac{1}{c\lambda t}&t\leq s\\\frac{1}{\lambda t}&o.w.\end{array}\right.\]
We have
\begin{align*}
&\max_{\q^*\in\Delta}\frac{1}{T}\sum_t^T E[\LL(\q^*,\theta_t) - \LL(\q_t,\theta_t)]\\
& \leq \frac{1}{T}(\frac{\mu^2}{2\lambda}(\log(T)+1) - \beta)
\end{align*}
where $\beta$ is a non-negative constant as $\beta = (\mu\sqrt{\frac{\log(s)}{2\lambda}} - \sqrt{s\lambda})^2$
and
$c = \frac{\mu}{\lambda}\sqrt{\frac{\log(s)}{2s}}$ should be larger than $1$.
\end{thm} 

Theorem~\ref{thm:converge} shows a $\mathcal{O}(\log(T)/T)$ convergence rate for the adversarial distribution. The gain of the adaptive learning rate is indicated in $\beta$, that is, a larger $\beta$ provides a better convergence. When applying the conventional learning rate i.e. $c=1$, it is easy to show $\beta = 0$. To further investigate the properties of $\beta$, we let $h(s) = \mu\sqrt{\log(s)/(2\lambda)} - \sqrt{s\lambda}$, i.e., $\beta = h(s)^2$, and study its behavior. 

\begin{prop}
$h(s)$ is non-negative.
\end{prop}
\begin{proof}
Since $c = \frac{\mu}{\lambda}\sqrt{\frac{\log(s)}{2s}}>1$, we have $\mu >\lambda\sqrt{\frac{2s}{\log(s)}}$. Therefore
\[h(s) = \mu\sqrt{\frac{\log(s)}{2\lambda}} - \sqrt{s\lambda}>\sqrt{s\lambda} - \sqrt{s\lambda}=0\]
\end{proof}

It implies that we can benefit from the variance reduction as long as the variance $\mu$ is sufficiently large. Then, we fix $\lambda=1$ and plot the curve of $h(s)$ when varying $\mu$ in Fig.~\ref{fig:curve}. We can find that $h(s)$ achieves its maximum after thousands of iterations, which suggests that $s$ should not be too large. It is consistent with our claim that the gradient will be shrunk by the learning rate and the additional variance has little influence when $t$ is large.

\begin{figure}[!ht]
\centering
\includegraphics[width= 2.2in ]{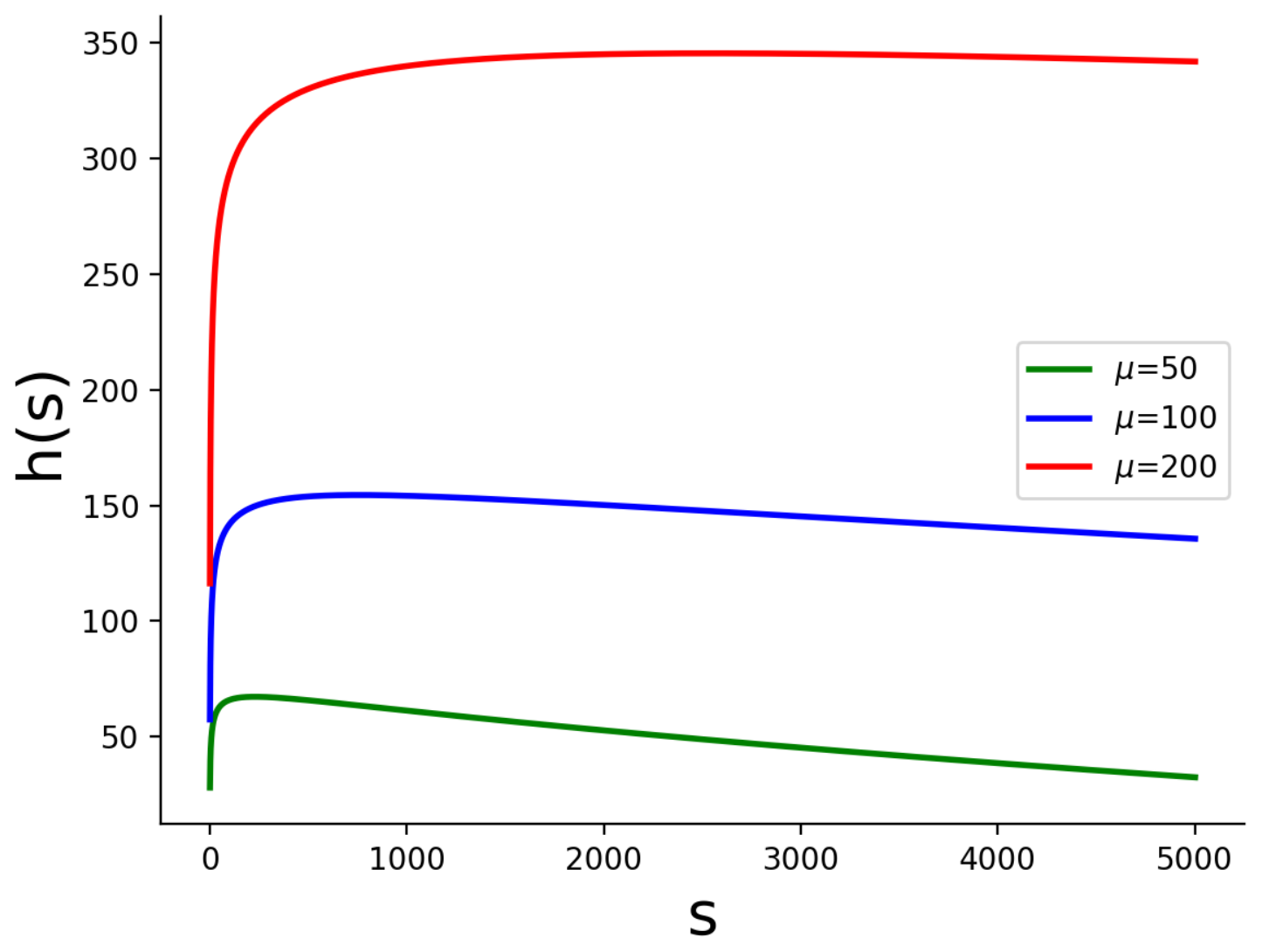}
\caption{Curves of $h(s)$ with different $\mu$'s.}\label{fig:curve}
\end{figure}

\section{Experiments}
\label{sec:exp}

We adopt ImageNet ILSVRC12~\cite{ILSVRC15} as the benchmark data set to learn models for generic feature extraction in the experiments. ImageNet includes $1,000$ classes, where each class has about $1,200$ images in training and $50$ images in test. We summarize these $1,000$ classes into $11$ concepts according to the structure of WordNet~\cite{fellbaum1998} that is the default class structure of ImageNet. The statistics of the concepts and classes are summarized in Table~\ref{tab:info}. Apparently, ImageNet is biased to specific animals. For example, it contains $59$ classes of birds and more than $100$ classes of dogs. This bias can result in the performance degeneration when applying the model learned by ERM to generate representations for different tasks. 

\begin{table}[!ht]
\centering
\begin{tabular}{|c||c|c|c|c|c|c|}
\hline
Concept&An&Ar&B&C&De&Do \\\hline
\#Classes&121&107&59&56&129&118\\\hline
Concepts&I&M&S&V&O&\\\hline
\#Classes&106&100&57&67&80&\\\hline
\end{tabular}
\caption{Concepts in ImageNet. The initials ``An'', ``Ar'', ``B'', ``C'', ``De'', ``Do'', ``I'', ``M'', ``S'', ``V'', ``O'' denote ``Animal'', ``Artifact'', ``Bird'', ``Container'', ``Device'', ``Dog'', ``Instrumentality'', ``Mammal'', ``Structure'', ``Vehicle'', ``Others'', respectively.}\label{tab:info}
\end{table}

We apply ResNet-18~\cite{HeZRS16}, which is a popular network as the feature extractor~\cite{mormont2018comparison}, to learn the representations. We train the model with stochastic gradient descent (SGD) on 2 GPUs. Following the common practice~\cite{HeZRS16}, we learn the model with $90$ epochs and set the size of mini-batch as $256$. The initial learning rate is set to $0.1$, and then it is decayed by a factor of $10$ at $\{30,60\}$. The weight decay is $10^{-4}$ and the momentum in SGD is $0.9$. All model training includes random crop and horizontal flipping as the data augmentation. We set $s=1000$ as suggested by Fig.~\ref{fig:curve} for the proposed algorithm. For the setting of $c$, we calculate the variance for $\mu$ from several mini-batches and set $c=10$ according to Theorem~\ref{thm:converge}. After obtaining deep models, we extract deep features from the layer before the last fully-connected layer, which generates a $512$-dimensional feature for a single image. Given the features, we learn a linear SVM~\cite{CC01a} to categorize examples. $\tau$, $\lambda$ and the parameter of SVM are searched in $\{10^i\} (i=-3,\ldots,1)$. Four different deep features with SVM as follows are compared in the experiments, where SVM$_{\mathrm{ERM}}$ is the conventional way to extract features with models trained by ERM and the others are our proposals.
\begin{compactitem}
\item SVM$_{\mathrm{ERM}}$: deep features learned with ERM.
\item SVM$_{\mathrm{EL}}$: deep features learned with example-level robustness only.
\item SVM$_{\mathrm{CL}}$: deep features learned with concept-level robustness only.
\item SVM$_{\mathrm{HRRL}}$: deep features learned with both example-level and concept-level robustness.
\end{compactitem}
Experiments are repeated $3$ times and the average results with standard deviation are reported. 

\subsection{CIFAR-10}
First, we study the scenario when example-level distribution difference exits between the target task and the benchmark data set. We conduct experiments on CIFAR-10~\cite{Krizhevsky2009Learning}, which contains $10$ classes and $60,000$ images. $50,000$ of them are for training and the rest are for test. CIFAR-10 has the similar concepts as those in ImageNet, e.g., ``bird'', ``dog'', and the difference in concepts is negligible. On the contrary, each image in CIFAR-10 has a size of $32\times32$, which is significantly smaller than that of images in ImageNet. As shown in Fig.~\ref{fig:example}, the example-level distribution changes dramatically and the example-level robustness is important for this task.

Table~\ref{tab:cifar} summarizes the comparison. First, we observe that the accuracy of SVM$_{\mathrm{ERM}}$ can achieve $85.77\%$, which surpasses the performance of SIFT features~\cite{BoRF10}, i.e., $65.6\%$, by more than $20\%$. It confirms that representations extracted from a DNN model trained on the benchmark data set can be applicable for generic tasks. Compared with representations from the model learned with ERM, SVM$_{\mathrm{EL}}$ outperforms it by a margin about $1\%$. It shows that optimizing with Wasserstein ambiguity set can learn the example-level robust features and handle the difference in examples better than ERM. SVM$_{\mathrm{CL}}$ has the similar performance as SVM$_{\mathrm{ERM}}$. It is consistent with the fact that the difference of concepts between CIFAR-10 and ImageNet is small. Finally, the performance of SVM$_{\mathrm{HRRL}}$ is comparable to that of SVM$_{\mathrm{EL}}$ due to negligible concept-level distribution difference but it is significantly better than SVM$_{\mathrm{ERM}}$, which demonstrates the effectiveness of the proposed algorithm.

\begin{table}[!ht]
\centering
\begin{tabular}{|l||c|}
\hline
Methods&Acc(mean$\pm$std) \\\hline
SVM$_{\mathrm{ERM}}$&85.77$\pm$0.12  \\\hline
SVM$_{\mathrm{EL}}$&86.62$\pm$0.18\\\hline
SVM$_{\mathrm{CL}}$&85.64$\pm$0.26\\\hline
SVM$_{\mathrm{HRRL}}$&86.49$\pm$0.19\\\hline
\end{tabular}
\caption{Comparison of accuracy ($\%$) on CIFAR-10.}\label{tab:cifar}
\end{table}

\subsection{Stanford Online Products (SOP)}
In this subsection, we demonstrate the importance of concept-level robustness. We have Stanford Online Products (SOP)~\cite{songCVPR16} as the target task to evaluate the learned representations. SOP collects product images from eBay.com and consists of $59,551$ images for training and $60,502$ images for test. We adopt the super class label for each image, which leads to a $12$-class classification problem. As shown in Fig.~\ref{fig:example}, we can find that the example-level distribution difference is not significant (e.g., resolution), while the distribution of concepts (i.e., concept-level distribution) is relatively different. ImageNet includes many natural objects, e.g., animals, while SOP only contains artificial ones. Handling the difference in concepts is challenging for this task.

Table~\ref{tab:sop} shows the performance comparison. Apparently, SVM$_{\mathrm{EL}}$ has the similar performance as SVM$_{\mathrm{ERM}}$ due to the minor changes in the example-level distribution. However, SVM$_{\mathrm{CL}}$ demonstrates a better accuracy, which is about $1\%$ better than SVM$_{\mathrm{ERM}}$. It demonstrates that the deep features learned with the proposed algorithm is more robust than those from ERM when the distribution of concepts varies. Besides, the performances of SVM$_{\mathrm{HRRL}}$ and SVM$_{\mathrm{CL}}$ are comparable, which confirms that deep features obtained with hierarchical robustness work well consistently in different scenarios.

\begin{table}[!ht]
\centering
\begin{tabular}{|l||c|}
\hline
Methods&Acc(mean$\pm$std) \\\hline
SVM$_{\mathrm{ERM}}$&73.47$\pm$0.09  \\\hline
SVM$_{\mathrm{EL}}$&73.48$\pm$0.08\\\hline
SVM$_{\mathrm{CL}}$&74.34$\pm$0.05\\\hline
SVM$_{\mathrm{HRRL}}$&74.23$\pm$0.08\\\hline
\end{tabular}
\caption{Comparison of accuracy ($\%$) on SOP.}\label{tab:sop}
\end{table}

\subsection{Street View House Numbers (SVHN)}

Finally, we deal with a task when both example-level and concept-level distribution differences exist. We evaluate the robustness of deep features on Street View House Numbers (SVHN)~\cite{netzer2011reading} data set. It consists of $73,257$ images for training and $26,032$ for test. The target is to identify one of $10$ digits from each $32\times 32$ image. The image has the same size as CIFAR-10, which is very different from ImageNet. Moreover, SVHN has the concepts of digits, which is also different from ImageNet.

We compare the different deep features in Table~\ref{tab:svhn}. First, as observed in CIFAR-10, SVM$_{\mathrm{EL}}$ outperforms SVM$_{\mathrm{ERM}}$ by a large margin. It is because features learned with example-level robustness is more applicable than those from ERM when examples are from a different distribution. Second, SVM$_{\mathrm{CL}}$ improves the performance by more than $2\%$. It is consistent with the observation in SOP, where features learned with concept-level robustness perform better when concepts vary. Besides, we can observe that the performance of SVM$_{\mathrm{CL}}$ surpasses that of SVM$_{\mathrm{EL}}$. It implies that controlling concept-level robustness, which has not been investigated sufficiently, may be more important than example-level robustness for representation learning. Finally, by combining example-level and concept-level robustness, SVM$_{\mathrm{HRRL}}$ shows an improvement of more than $4\%$. It demonstrates that example-level and concept-level robustness are complementary. Incorporating both of them can further improve the performance of deep features, when the example-level and concept-level distributions are different from these of the benchmark data set.

\begin{table}[!ht]
\centering
\begin{tabular}{|l||c|}
\hline
Methods&Acc(mean$\pm$std) \\\hline
SVM$_{\mathrm{ERM}}$&63.23$\pm$0.35  \\\hline
SVM$_{\mathrm{EL}}$&65.01$\pm$0.37\\\hline
SVM$_{\mathrm{CL}}$&65.47$\pm$0.27\\\hline
SVM$_{\mathrm{HRRL}}$&67.33$\pm$0.39\\\hline
\end{tabular}
\caption{Comparison of accuracy ($\%$) on SVHN.}\label{tab:svhn}
\end{table}

\begin{figure*}[!ht]
\centering
\begin{minipage}{2in} 
\centering
\includegraphics[height= 1.4 in ]{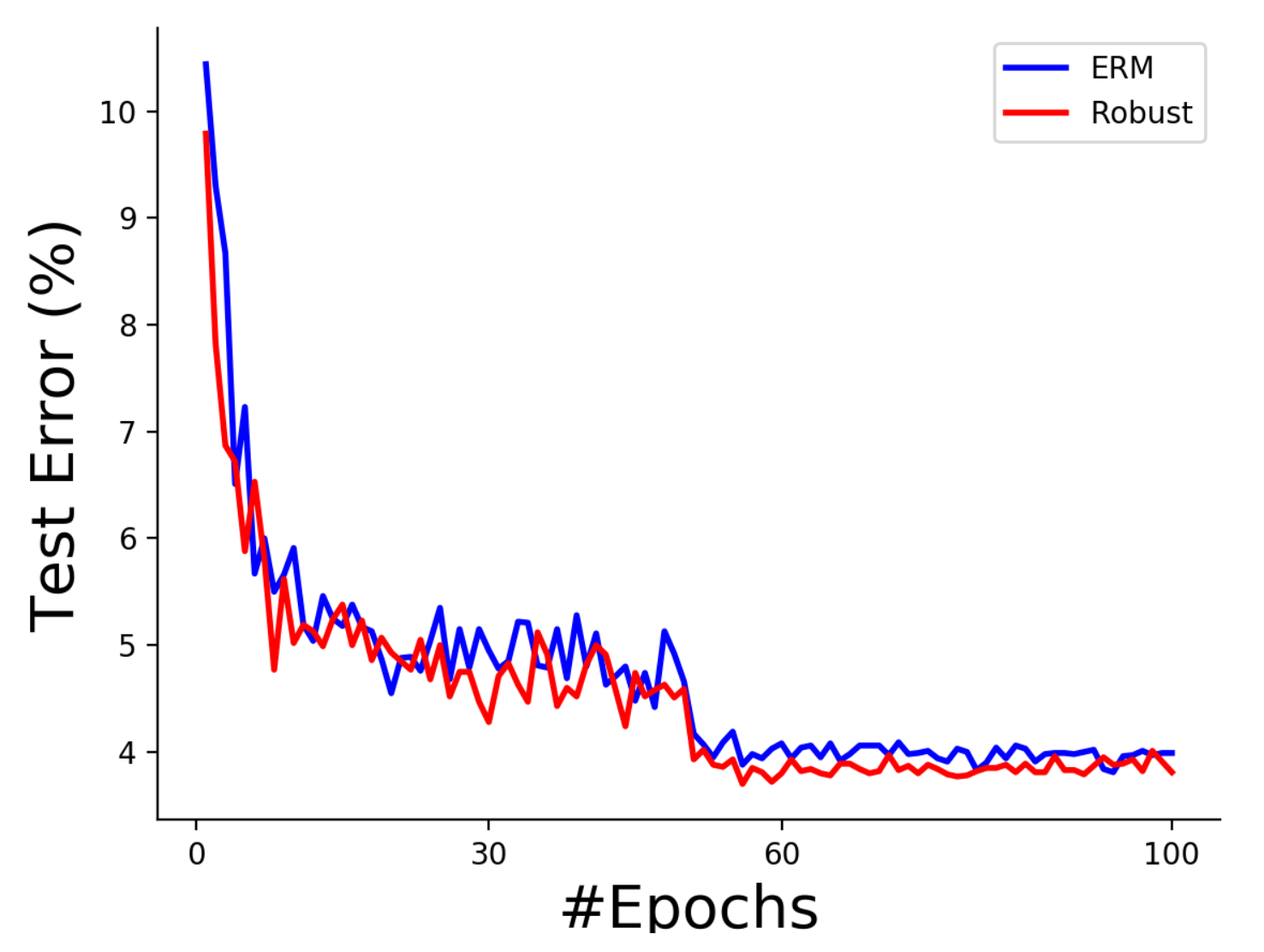}
\mbox{\footnotesize (a) CIFAR-10}
\end{minipage}
\begin{minipage}{2in} 
\centering
\includegraphics[height= 1.4in ]{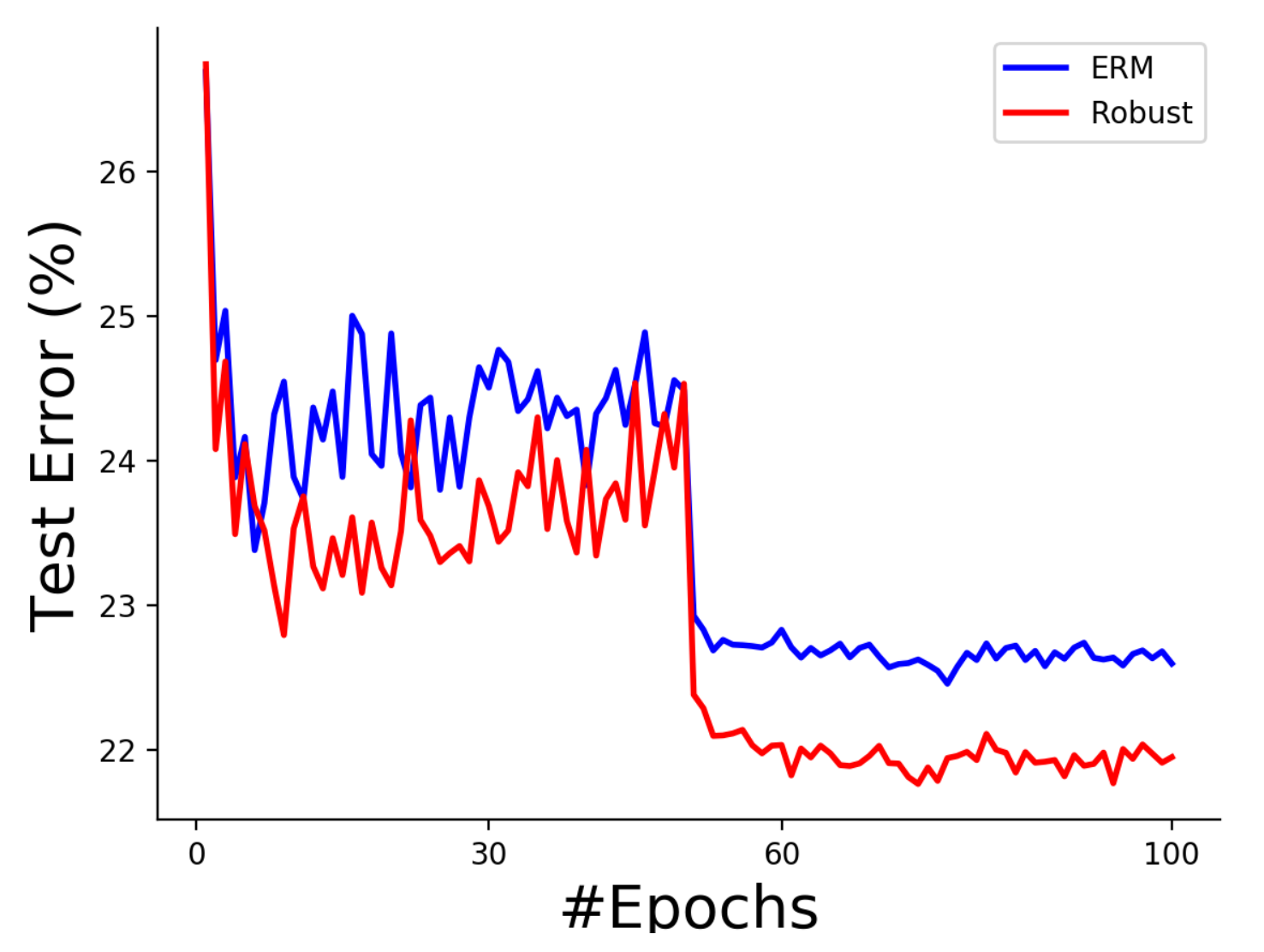}
\mbox{\footnotesize (b) SOP}
\end{minipage}
\begin{minipage}{2in} 
\centering
\includegraphics[height= 1.4in ]{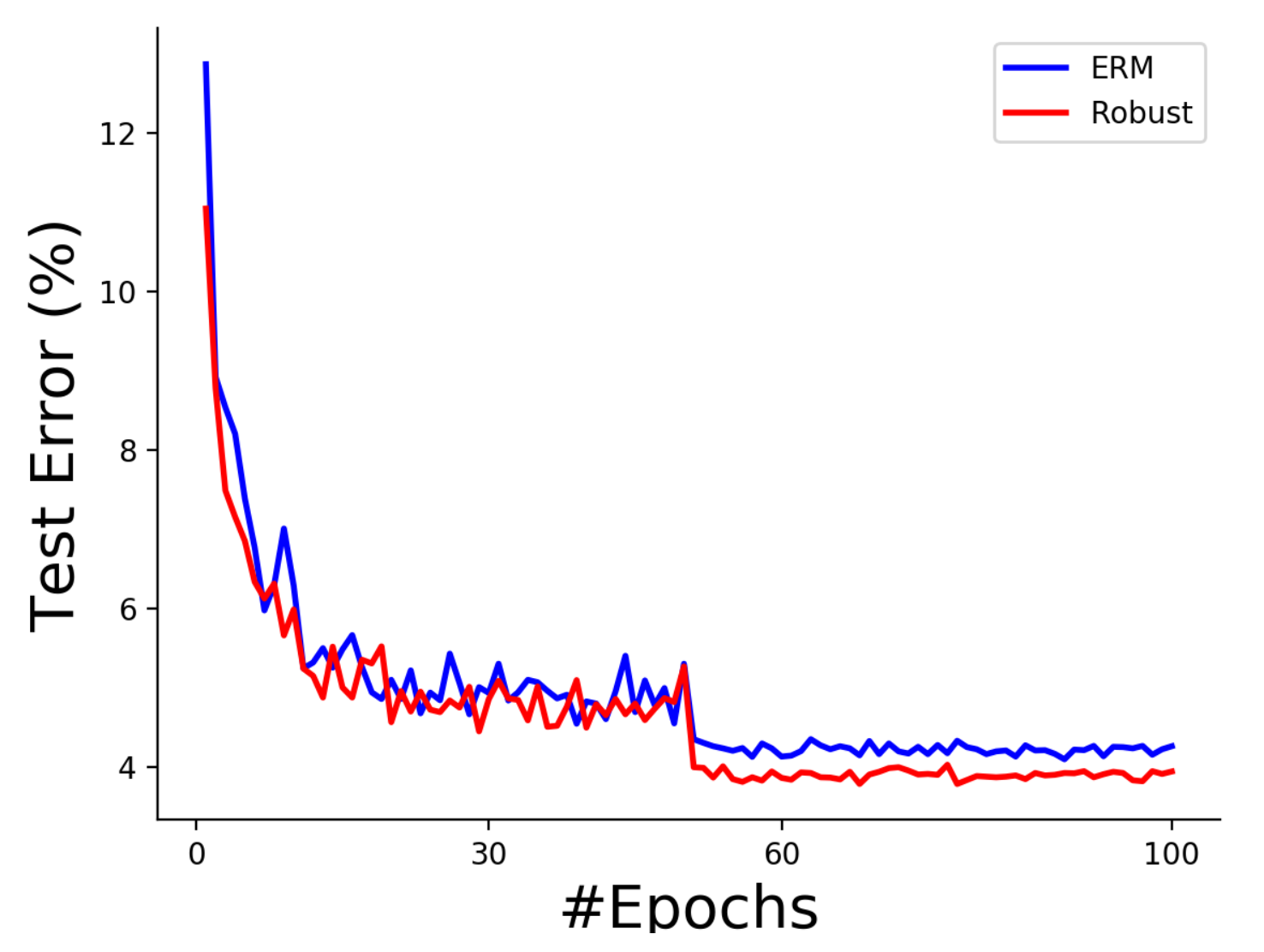}
\mbox{\footnotesize (c) SVHN}
\end{minipage}
\caption{Comparison of fine-tuning with different initializations.}\label{fig:ft}
\end{figure*}

\subsection{Fine-tuning}

Besides extracting features, a pre-trained model is often applied as an initialization for training DNNs on the target task when GPUs are available. Since initialization is crucial for the final performance of DNNs~\cite{SutskeverMDH13}, we conduct the experiments that initialize the model with parameters trained on ImageNet and then fine-tune the model on CIFAR-10, SOP and SVHN. After initialization, the model is fine-tuned with $100$ epochs, where ERM is adopted as the objective for each task. The learning rate is set as $0.01$ and decayed once by a factor of $10$ after $50$ epochs. Fig.~\ref{fig:ft} illustrates the curve of test error. We let ``ERM'' denote the model initialized with that pre-trained with ERM and ``Robust'' denote the one initialized with the model pre-trained with the proposed algorithm. Surprisingly, we observe that the models initialized with the proposed algorithm still surpass those with ERM. It implies that the learned robust models can be used for initialization besides feature extraction.

\subsection{Effect of Robustness}
Finally, we investigate the effect of the proposed method on ImageNet task itself to further illustrate the impact of robustness. First, we demonstrate the results of example-level robustness. We generate the augmented examples for validation set as in Theorem~\ref{thm:2} and report the accuracy of different models in Fig.~\ref{fig:el}. The horizontal axis shows the step size for generating the augmented examples. When step size is $0$, the original validation set is used for evaluation. Otherwise, each image in the validation set is modified with the corresponding step size, and only modified images are used for evaluation. Intuitively, larger step size implies larger example-level distribution change compared to the original ImageNet data set.

\begin{figure}[!ht]
\centering
\includegraphics[height= 1.55in ]{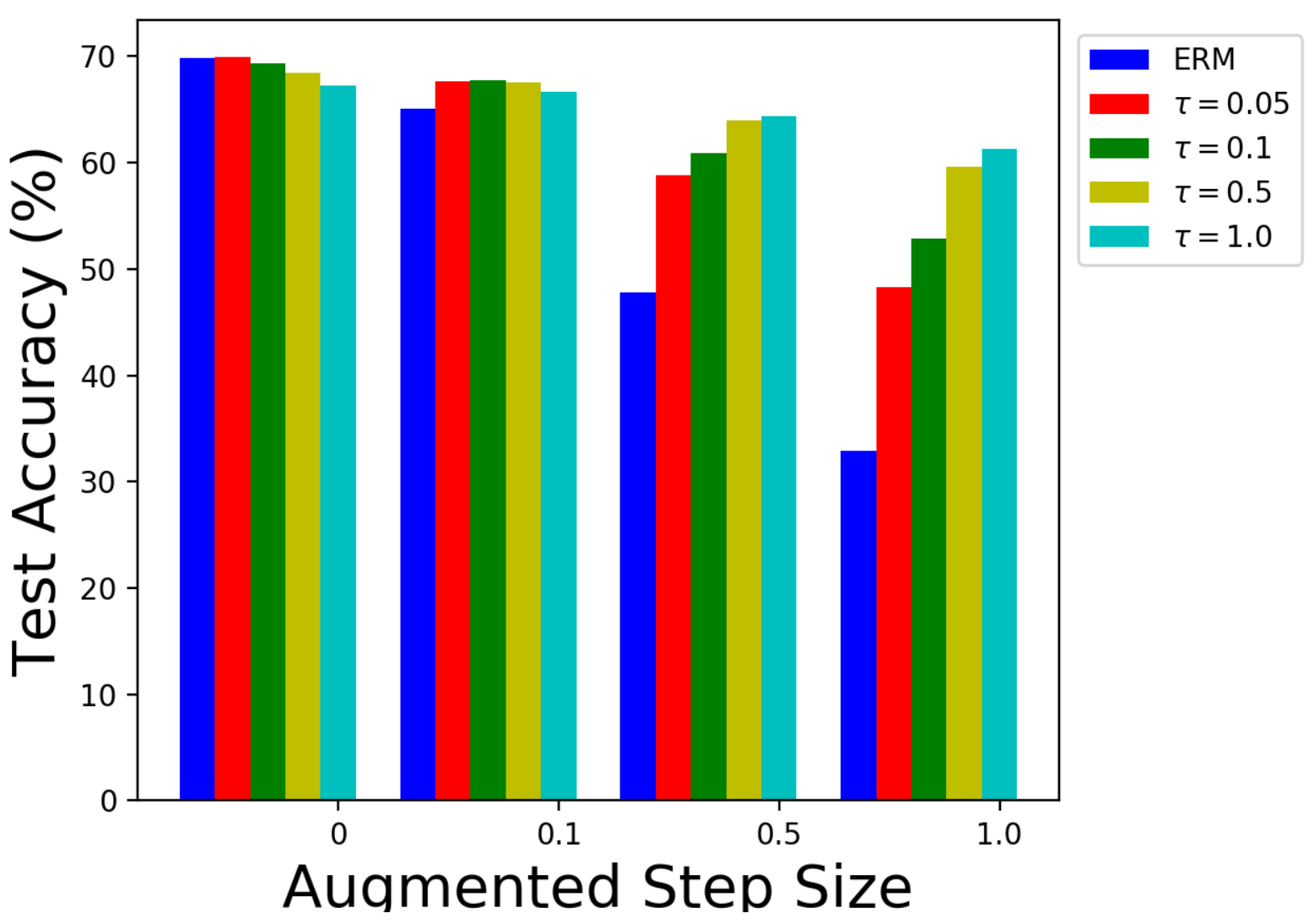}
\caption{Comparison of accuracy on augmented examples.}\label{fig:el}
\end{figure}

Besides ERM, four different models are included in the comparison. Each model is trained with the example-level robustness and the corresponding parameter $\tau$ is denoted in the legend, where larger $\tau$ should theoretically provide a more robust model.

We can observe that ERM performs well when there is no augmentation but its performance degrades significantly when the augmentation step size increases. It confirms that ERM cannot generalize well when the example-level distribution changes. Fortunately, we can observe that more robust models (i.e., $\tau$ increases) can provide better generalization performance as expected. It is because that the proposed algorithm focuses on optimizing the worst-case performance among different distributions derived from the original distribution.

Second, we show the influence of concept-level robustness. We train models with different $\lambda$ for regularization and summarize the accuracy of concepts in Fig.~\ref{fig:cl}. We sort the accuracy in ascending order to make the comparison clear. As illustrated, ERM aims to optimize the uniform distribution of examples and ignores the distribution of concepts. Consequently, certain concept, e.g., ``bird'', has much higher accuracy than others. When decreasing $\lambda$ in our proposed method, the freedom of adversarial distribution increases. With more freedom, the proposed method will focus on the concepts with bad performance. By optimizing the adversarial distribution, the model will balance the performance between different concepts as illustrated in Fig.~\ref{fig:cl}.

\begin{figure}[!ht]
\centering
\includegraphics[height= 1.55in ]{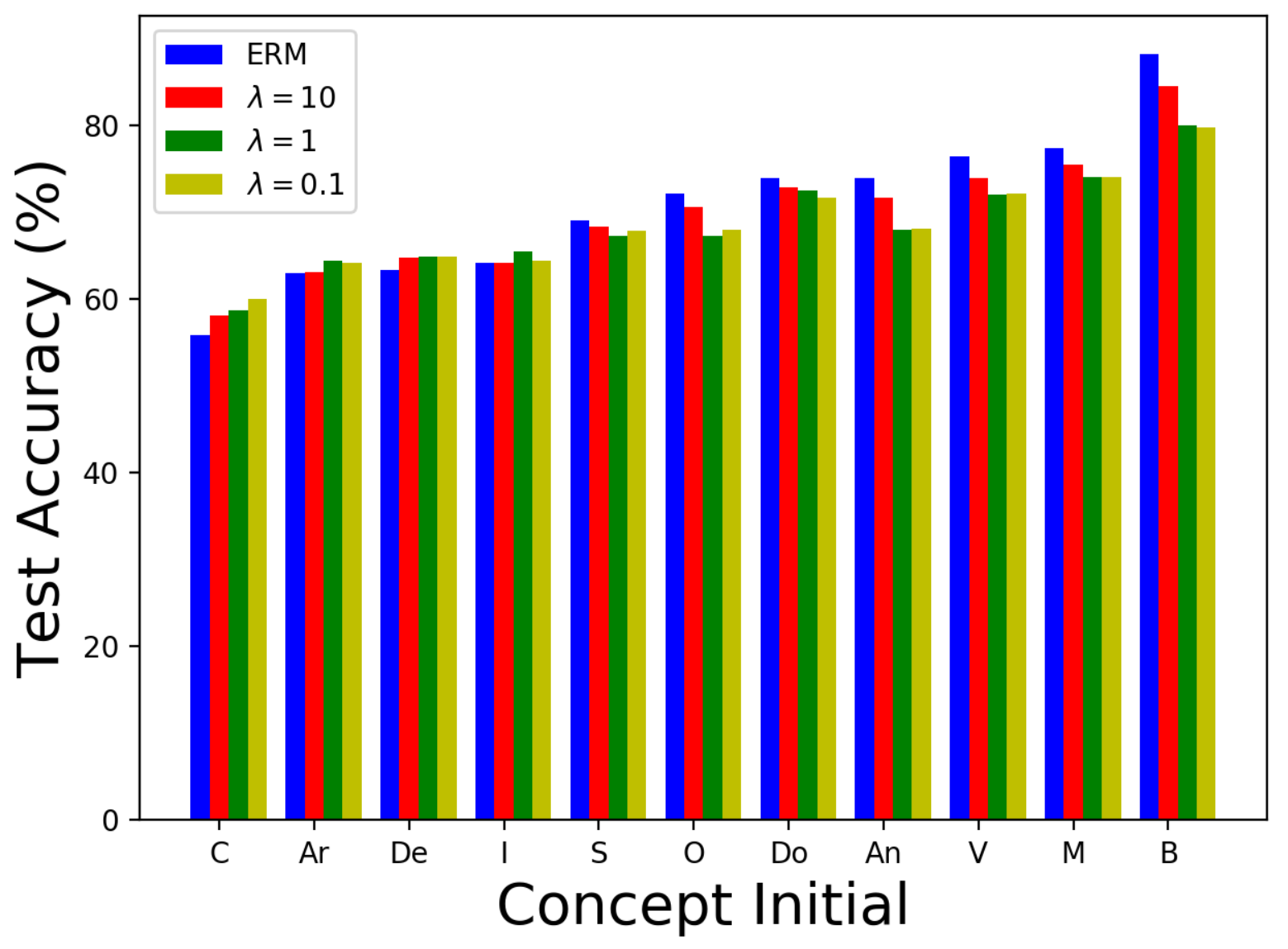}
\caption{Comparison of accuracy on concepts in ImageNet.}\label{fig:cl}
\end{figure}

In summary, Figs.~\ref{fig:el} and \ref{fig:cl} demonstrate the different influences of example-level and concept-level robustness. Evidently, our method can deal with the perturbation from different aspects. It further confirms that improving the hierarchical robustness is important for applying deep features or initializing models in real-world applications.

\section{Conclusion}
\label{sec:conclude}
In this work, we study the problem of learning deep features using a benchmark data set for generic tasks. We propose a hierarchically robust optimization algorithm to learn robust representations from a large-scale benchmark data set. The theoretical analysis also demonstrates the importance of augmentation when training DNNs. The experiments on real-world data sets demonstrate the effectiveness of the learned features from the proposed method. The framework can be further improved when side information is available. For example, given the concepts of the target domain, we can obtain the specific reference distribution $\q_0$ accordingly, and then learn the features for the desired task. This direction can be our future work.

\bibliographystyle{ieee_fullname}
\bibliography{rb_20}

\appendix

\section{Proof of Theorem~1}
\begin{proof}
Due to the smoothness, we have
\[\ell(\hat{\x}_i,y_i;\theta) \leq \ell(\x_i,y_i;\theta)+\langle \nabla_{\x_i}\ell,\hat{\x}_i-\x_i\rangle+\frac{L_{\x}}{2}\|\hat{\x}_i - \x_i\|_F^2\]
So
\begin{align*}
&\ell(\hat{\x}_i,y_i;\theta) - \frac{\lambda_w}{2} \|\hat{\x}_i - \x_i\|_F^2 \leq \ell(\x_i,y_i;\theta)+\langle \nabla_{\x_i}\ell,\hat{\x}_i-\x_i\rangle\\
&-\frac{\lambda_w-L_{\x}}{2}\|\hat{\x}_i - \x_i\|_F^2
\end{align*}
When $\lambda_w$ is sufficiently large as $\lambda_w>L_{\x}$, R.H.S. is bounded and
\begin{eqnarray*}
&&\ell(\hat{\x}_i,y_i;\theta) - \frac{\lambda_w}{2} \|\hat{\x}_i - \x_i\|_F^2 \\
&&\leq \ell(\x_i,y_i;\theta)+\frac{1}{2(\lambda_w - L_{\x})}\|\nabla_{\x_i}\ell\|_F^2
\end{eqnarray*}

Since $\nabla_{\x}\ell(\cdot)$ is $L_{\theta}$-Lipschitz continuous, we have
\begin{eqnarray*}
\|\nabla_{\x}\ell(\x;\theta)\|_F^2&\leq& 2\|\nabla_{\x}\ell(\x;\theta)-\nabla_{\x}\ell(\x;\mathbf{0})\|_F^2\\
&&+2\|\nabla_{\x}\ell(\x;\mathbf{0})\|_F^2\\
&\leq& 2L_{\theta}^2\|\theta\|_F^2+2\|\nabla_{\x}\ell(\x;\mathbf{0})\|_F^2
\end{eqnarray*}
Note that $\|\nabla_{\x}\ell(\x;\mathbf{0})\|_F=0$ in many convolutional neural networks. The bound can be improved and the original subproblem can be bounded as
\[\max_{\hat{\x}_i\in\X}\ell(\hat{\x}_i,y_i;\theta) - \frac{\lambda_w}{2} \|\hat{\x}_i - \x_i\|_F^2 \leq \ell(\x_i,y_i;\theta) +\frac{\gamma}{2}\|\theta\|_F^2 \]
where $\gamma = \frac{L_{\theta}^2}{\lambda_w -L_{\x}}$.
\end{proof}

\section{Proof of Theorem~2}
\begin{proof}
We consider the augmented examples as
\[\tilde{\x}_i = \x_i+\tau \z_i\]
According to the smoothness, we have
\begin{align*}
&\ell(\hat{\x}_i,y_i;\theta) - \frac{\lambda_w}{2}\|\hat{\x}_i-\x_i\|^2\leq \ell(\tilde{\x}_i,y_i;\theta) +\langle\nabla_{\tilde{\x}_i}\ell, \hat{\x}_i - \tilde{\x}_i \rangle\\
&+\frac{L_{\x}}{2}\|\hat{\x}_i - \tilde{\x}_i\|- \frac{\lambda_w}{2}\|\hat{\x}_i-\x_i\|^2\\
&= \ell(\tilde{\x}_i,y_i;\theta)+\langle\nabla_{\tilde{\x}_i}\ell - \tau L_{\x} \z_i, \hat{\x}_i - \x_i \rangle-\tau \langle\nabla_{\tilde{\x}_i}\ell, \z_i \rangle\\
&+\frac{L_{\x}\tau^2}{2}\|\z_i\|^2-\frac{\lambda_w - L_{\x}}{2}\|\hat{\x}_i - \x_i\|^2\\
&\leq \ell(\tilde{\x}_i,y_i;\theta)+\frac{\|\nabla_{\tilde{\x}_i}\ell-\tau L_{\x}\z_i \|_F^2}{2(\lambda_w - L_{\x})}- \tau\langle \nabla_{\tilde{\x}_i}\ell,\z_i\rangle\\
&+\frac{L_{\x}\tau^2}{2}\|\z_i\|_F^2\\
&= \ell(\tilde{\x}_i,y_i;\theta)+\frac{\|\nabla_{\tilde{\x}_i}\ell\|_F^2}{2(\lambda_w-L_{\x})} \\
&+\frac{\lambda_w}{\lambda_w-L_{\x}}(\frac{\tau^2L_{\x}\|\z_i\|_F^2}{2} - \tau \langle \nabla_{\tilde{\x}_i}\ell,\z_i\rangle)\\
&\leq\ell(\tilde{\x}_i,y_i;\theta)+\frac{\gamma}{2}\|\theta\|_F^2\\
&+\frac{\lambda_w}{\lambda_w-L_{\x}}(\frac{\tau^2L_{\x}\|\z_i\|_F^2}{2} - \tau \langle \nabla_{\tilde{\x}_i}\ell - \nabla_{\x_i}\ell,\z_i\rangle - \tau \langle \nabla_{\x_i}\ell,\z_i\rangle)\\
&\leq \ell(\tilde{\x}_i,y_i;\theta)+\frac{\gamma}{2}\|\theta\|_F^2 \\
&+\frac{\lambda_w}{\lambda_w-L_{\x}}(\frac{\tau^2L_{\x}\|\z_i\|_F^2}{2} + \tau \| \nabla_{\tilde{\x}_i}\ell - \nabla_{\x_i}\ell\|\|\z_i\| - \tau \langle \nabla_{\x_i}\ell,\z_i\rangle)\\
&\leq \ell(\tilde{\x}_i,y_i;\theta)+\frac{\gamma}{2}\|\theta\|_F^2 \\
&+\frac{\lambda_w}{\lambda_w-L_{\x}}(\frac{3\tau^2L_{\x}\|\z_i\|_F^2}{2} - \tau \langle \nabla_{\x_i}\ell,\z_i\rangle)\\
&=\ell(\tilde{\x}_i,y_i;\theta)+\frac{\gamma}{2}\|\theta\|_F^2 - \frac{\lambda_w}{\lambda_w-L_{\x}}\frac{\langle \nabla_{\x_i}\ell,\z_i\rangle^2}{6L_{\x}\|\z_i\|_F^2}
\end{align*}
The last equation is from setting $\tau$ to optimum as
\[\tau = \frac{\langle \nabla_{\x_i}\ell,\z_i\rangle}{3L_{\x}\|\z_i\|_F^2}\]
\end{proof}

\section{Proof of Theorem~3}
\begin{proof}
For an arbitrary distribution $\q$, we have
\begin{align*}
&E[\|\q_{t+1}-\q\|_2^2] = E[\|\P_{\Delta}(\q_t+\eta_t g_t) - \q\|_2^2]\\
&\leq E[\|\q_t+\eta_tg_t - \q\|_2^2]\\
&=E[\|\q_t-\q\|_2^2+2\eta_t (\q_t-\q)^\top g_t+\eta_t^2\|g_t\|_2^2]\\
&\leq E[\|\q_t-\q\|_2^2+\eta_t^2\mu^2\\
&+2\eta_t(\LL(\q_t,\theta_t) - \LL(\q,\theta_t) -\frac{\lambda}{2}\|\q_t-\q\|_2^2)]\\
\end{align*}
The last inequality is from the fact that the objective is $\lambda$-strongly concave in $\q$ and the observed gradient is unbiased.
Therefore, we have
\begin{align*}
& E[\LL(\q,\theta_t) - \LL(\q_t,\theta_t)]\leq \frac{ E[\|\q_t-\q\|_2^2] -  E[\|\q_{t+1}-\q\|_2^2]}{2\eta_t} \\
&-\frac{\lambda}{2}\|\q_t-\q\|_2^2+\frac{\eta_t}{2}\mu^2
\end{align*}
When $\eta_t = \frac{1}{\lambda t}$, we have
\begin{align*}
& E[\LL(\q,\theta_t) - \LL(\q_t,\theta_t)]\leq \frac{\lambda t}{2} (E[\|\q_t-\q\|_2^2] -  E[\|\q_{t+1}-\q\|_2^2]) \\
&-\frac{\lambda}{2}\|\q_t-\q\|_2^2+\frac{1}{2\lambda t}\mu^2
\end{align*}
When $\eta_t = \frac{1}{\lambda t c}$, we have
\begin{align*}
& E[\LL(\q,\theta_t) - \LL(\q_t,\theta_t)]\leq \frac{\lambda t c}{2} (E[\|\q_t-\q\|_2^2] -  E[\|\q_{t+1}-\q\|_2^2]) \\
&-\frac{\lambda}{2}\|\q_t-\q\|_2^2+\frac{1}{2\lambda t c}\mu^2
\end{align*}

We assume that $\eta_t = \frac{1}{c\lambda t}$ and $c>1$ for the first $s$ iterations and then $\eta_t = \frac{1}{\lambda t}$. So we have
\begin{align*}
&\sum_t^T E[\LL(\q,\theta_t) - \LL(\q_t,\theta_t)] = \sum_{t=1}^s E[\LL(\q,\theta_t) - \LL(\q_t,\theta_t)]\\
&+\sum_{t=s+1}^TE[\LL(\q,\theta_t) - \LL(\q_t,\theta_t)]\\
&\leq \sum_{t=1}^s \big((\frac{c\lambda}{2} - \frac{\lambda}{2})E[\|\q_t-\q\|_2^2]+\frac{1}{2\lambda tc }\mu^2\big)+\sum_{t=s+1}^T \frac{1}{2\lambda t}\mu^2\\
&\leq s\lambda(c-1)+\frac{\mu^2}{2\lambda}\log(s) (\frac{1}{c}-1)+\frac{\mu^2}{2\lambda}(\log(T)+1)
\end{align*}
By setting $c = \frac{\mu}{\lambda}\sqrt{\frac{\log(s)}{2s}}$ and $\q$ to be optimum, we have
\begin{align*}
&\max_{\q^*\in\Delta}\sum_t^T E[\LL(\q^*,\theta_t) - \LL(\q_t,\theta_t)] \leq \mu\sqrt{2s\log(s)} - s\lambda \\
&- \frac{\mu^2\log(s)}{2\lambda}+\frac{\mu^2}{2\lambda}(\log(T)+1)\\
&= \frac{\mu^2}{2\lambda}(\log(T)+1) - (\mu\sqrt{\frac{\log(s)}{2\lambda}} - \sqrt{s\lambda})^2
\end{align*}

\end{proof}

\end{document}